\documentclass[12pt,draftcls,onecolumn]{IEEEtran}
\IEEEoverridecommandlockouts
\usepackage{balance} 
\usepackage{enumerate}
\usepackage{algorithm}
\usepackage[noend]{algpseudocode}
\usepackage{amsmath}
\usepackage{amsthm}
\usepackage{amsmath}
\usepackage{amssymb}
\usepackage{amsthm}
\usepackage{graphicx}
\usepackage{graphics}
\usepackage{mathtools}
\usepackage{multirow}
\usepackage[export]{adjustbox}
\usepackage[document]{ragged2e}
\usepackage{enumerate}
\theoremstyle{plain}

\newtheorem{theorem}{Theorem}

\newcommand{\R}{\mathbb{R}}
\newcommand{\E}{\mathbb{E}}
\newcommand{\F}{\mathcal{F}}

\usepackage{caption}
\usepackage{subcaption}
\usepackage{tikz}
\usepackage{url}
\usepackage[fleqn]{nccmath}
\usepackage{cite}

\usepackage{tikz}
\usepackage{textcomp}
\usepackage{hyperref}
\usepackage{lipsum}

\newcommand\copyrighttext{%
  \footnotesize \textcopyright 2022 IEEE. Personal use is permitted, but republication/redistribution requires IEEE permission. A version of this paper has been accepted for presentation at the International Joint Conference on Neural Networks (IJCNN) at IEEE WCCI 2022 and for publication in the conference proceedings published by IEEE.}
\newcommand\copyrightnotice{%
\begin{tikzpicture}[remember picture,overlay]
\node[anchor=south,yshift=10pt] at (current page.south) {\fbox{\parbox{\dimexpr\textwidth-\fboxsep-\fboxrule\relax}{\copyrighttext}}};
\end{tikzpicture}%
}

\title{Neural Network Compatible \\ Off-Policy Natural Actor-Critic Algorithm}

\author{Raghuram Bharadwaj Diddigi$^{*}$, Prateek Jain$^{*}$, Prabuchandran K.J., Shalabh Bhatnagar
\thanks{$^{*}$ Equal Contribution.}
\thanks{R. B. Diddigi and S. Bhatnagar are with the Department of Computer
Science and Automation, Indian Institute of Science, Bengaluru
India (e-mails: raghub@iisc.ac.in; shalabh@iisc.ac.in).}%
\thanks{Prateek Jain and Prabuchandran K.J. are with the Department of Computer
Science and Engineering, Indian Institute of Technology Dharwad, India (e-mails: 170010007@iitdh.ac.in; prabukj@iitdh.ac.in).}%
\thanks{Raghuram Bharadwaj was supported by a fellowship grant from the Centre for Networked Intelligence (a Cisco CSR initiative) of the Indian Institute of Science, Bangalore. Prabuchandran K.J. was supported by the Science and Engineering Board (SERB), Department of Science and Technology, Government of India for the startup research grant `SRG/2021/000048'. Shalabh Bhatnagar was supported by the J.C.Bose Fellowship, a project from DST under the ICPS Program and the RBCCPS, IISc.}
}

\begin{document}
\maketitle
\thispagestyle{empty}
\pagestyle{empty}
\copyrightnotice


\begin{abstract}
Learning optimal behavior from existing data is one of the most important problems in Reinforcement Learning (RL). This is known as ``off-policy control'' in RL where an agent's objective is to compute an optimal policy based on the data obtained from the given policy (known as the behavior policy). As the optimal policy can be very different from the behavior policy, learning optimal behavior is very hard in the ``off-policy'' setting compared to the ``on-policy'' setting where new data from the policy updates will be utilized in learning. This work proposes an off-policy natural actor-critic algorithm that utilizes state-action distribution correction for handling the off-policy behavior and the natural policy gradient for sample efficiency. The existing natural gradient-based actor-critic algorithms with convergence guarantees require fixed features for approximating both policy and value functions. This often leads to sub-optimal learning in many RL applications. On the other hand, our proposed algorithm utilizes compatible features that enable one to use arbitrary neural networks to approximate the policy and the value function and guarantee convergence to a locally optimal policy.  We illustrate the benefit of the proposed off-policy natural gradient algorithm by comparing it with the vanilla gradient actor-critic algorithm on benchmark RL tasks.

\end{abstract}



\section{Introduction}\label{intro}

Reinforcement Learning (RL) \cite{bertsekas1996neuro,sutton1998introduction} deals with learning algorithms for an agent to learn optimal behavior based on the feedback received from an unknown environment. 
In a standard RL setting, the agent learns a sequence of actions, executes these actions, receives feedback from the environment, and improves its ability to take right actions. This cycle continues until a best sequence of actions is learnt. Note that, in each cycle, the agent gets a feedback from the environment for its current choice of actions. In many real-life scenarios like medical research \cite{liu2020off}, this paradigm is not desired. For example, consider training a RL model to learn the best course of treatment for combating a disease. In this application, it is dangerous to try exploratory actions as done in the standard RL paradigm. The most appropriate paradigm here would be to train the RL algorithm from the past patient records where certain treatments have already been implemented. Of course, the efficacy of the trained model depends on the exhaustiveness of the data, but this approach is safer to learn a basic model that can be used as a guide for a doctor. This is referred as the off-policy setting in the RL literature. 

Two sub-problems are studied under the off-policy setting - prediction and control. Off-policy prediction problem \cite{liu2018breaking, sutton2009fast,sutton2016emphatic,bhatnagar2009convergent,maei2010gq,gelada2019off} refers to computing the value function of a given policy \cite{sutton1998introduction} (known as the target policy) from the data obtained from another policy (known as the behavior policy). On the other hand, the goal in the off-policy control problem is to compute optimal policy from the data obtained from the behavior policy. In this work, we consider the off-policy control problem. We first discuss the on-policy setting and subsequently discuss the off-policy setting. 

Standard actor-critic algorithms \cite{konda2000actor,sutton1999policy,mnih2016asynchronous} are a popular class of on-policy RL algorithms for computing an optimal policy. In this paradigm, the actor improves the policy based on the feedback from the critic using the policy gradient rule. The gradient updates in these algorithms are highly sensitive to the learning rates, and it is often difficult to \textit{a priori} know the right learning rate schedule for an efficient update of parameters. This forces one to perform an extensive hyper-parameter search to arrive at a good solution. To alleviate this problem, natural gradients have been studied in the literature \cite{kakade2001natural,peters2008natural}. 

The idea in natural gradients is to move in the right direction in the distribution space induced by the function parameters rather than in the parameter space. As a result, the natural gradient algorithms are not sensitive to model reparameterization. Moreover, these algorithms are shown to be more efficient and faster compared to standard gradient algorithms \cite{zhang2019fast,peters2008natural}. The actor-critic algorithms that make use of natural gradients are commonly referred to as natural actor-critic algorithms \cite{kakade2001natural,peters2003reinforcement,peters2008natural,bhatnagar2009natural}. Here, the actor uses natural gradients instead of normal gradients in the update rule of policy parameters. In \cite{iwaki2019implicit}, an incremental on-policy natural actor-critic algorithm has been proposed and is shown to improve the stability of the iterates. The trust region optimization techniques for RL like the TRPO \cite{schulman2015trust} and the ACKTR \cite{wu2017scalable} use natural policy gradient in their updates. However, they explicitly compute/approximate a distance metric using the Fisher information matrix (curvature information) giving rise to complex algorithms. This estimation/approximation of the Fisher information matrix can be entirely avoided (see equation \eqref{eq12} of Section \ref{s3}) if the critic employs a linear function approximator with well-defined features known as compatible features \cite{konda2000actor}. 

The idea of using compatible features in critics has been first proposed in \cite{konda2000actor,sutton1999policy} in the context of standard policy gradient algorithms. Actor-critic algorithms using compatible features in critics also enjoy nice convergence properties. However, the use of compatible features forces the critic to employ only the linear function approximation architecture. This constraints the computational capability of critic to accurately estimate the feedback needed for the actor which becomes a bottleneck, especially in environments with large/continuous state spaces. Further, utilising compatible features along with neural network  based implementations of actor-critic algorithms is a challenging task. 
Therefore, typically in practical implementations, actor and critic use different neural networks to approximate the policy and value function which might lead to undesired behaviors. The above mentioned limitations have 
restricted the development of successful neural network compatible natural actor-critic algorithms. For ease of exposition, we refer neural network compatible natural actor-critic algorithms as deep natural actor-critic algorithms. 

In our work, we alleviate aforementioned problems by constructing our algorithm with three components - (1) an actor that deploys a deep neural network for approximating the policy, (2) an advantage critic that uses linear function approximation with compatible features for estimating the advantage value function \cite{sutton1998introduction} (which we refer to as neural network compatible features), and (3) a value function critic that deploys deep neural network for approximating the value function. To the best of our knowledge, ours is the first work that allows compatible features to be used in conjunction with the neural network approximations of policy and value function. Our first contribution in this work is a ``Deep Natural Actor-Critic'' with compatible features and its comparative study with the standard deep actor-critic algorithm. Note that our algorithm implements natural policy gradient in a simple manner while utilizing the power of neural networks.

We now describe the ``off-policy'' problem. As mentioned earlier, this refers to the setting where the training data comes from a policy (known as behavior policy) which can be very different from the optimal policy. As a result, the state-action distribution between the Markov Chain induced by behavior policy and the learning policy will be different. This mismatch in distributions is what makes the off-policy training unstable. 
In \cite{degris2012off}, the first off-policy actor-critic has been proposed and is shown to converge under tabular setting. Since then, many off-policy actor-critic algorithms have been proposed mainly under two objective criteria: Excursion and Alternate Life \cite{zhang2019generalized}. Under the excursion objective criterion, the performance of a policy is evaluated under the expectation with respect to the stationary distribution of the behavior policy. The algorithms proposed in \cite{degris2012off,zhang2020provably,imani2018off} fall under this category. However, it is shown in \cite{zhang2019generalized} that this setting can be misleading about the performance of optimal policy. Under alternate life objective criterion, the performance of a policy is evaluated with respect to the stationary distribution of the learning policy itself. This seems to be the most natural setting for off-policy control. The algorithms proposed in \cite{gelada2019off,liu2020off,xu2021doubly} fall under this category. 

In \cite{zhang2019generalized}, a unified objective combining both excursion and after-life objectives has been considered and is shown to provide better performance. Recently, in \cite{khodadadian2021finite}, an off-policy version of natural actor-critic has been proposed under tabular setting and finite sample analysis is provided. In \cite{liu2020off}, a standard actor-critic algorithm has been proposed for solving the off-policy control problem. The actor in the algorithm of \cite{liu2020off} corrects for the discrepancy in state and action distribution between target and behavior policy while updating the policy parameters. However a similar state distribution correction is not performed for the value function update, which might lead to divergence \cite{sutton1998introduction}.
Our second contribution is an off-policy natural actor-critic algorithm under the alternate life objective criterion which corrects the state-action mismatch in actor and critic updates. To the best of our knowledge, ours is the first off-policy natural actor-critic algorithm that utilizes non-linear or deep function approximation. We now summarize the main contributions of our work as follows:
\begin{enumerate}
    \item We propose a new deep actor-critic algorithm under the on-policy setting with two critics and an actor. Here, the actor employs a natural policy gradient, the advantage critic uses compatible features, and the value critic uses non-linear function approximation. 
    \item We then extend this algorithm to the off-policy setting by appropriately correcting the state-action distribution.
    \item We successfully integrate compatible features in the implementation of our proposed algorithms and demonstrate their advantages through comparisons with standard actor-critic algorithm on benchmark RL tasks.
\end{enumerate}

Finally, it is important to observe the distinction between the off-policy setting that we consider in this paper, and the actor-critic algorithms like Soft Actor-Critic (SAC) \cite{haarnoja2018soft} and Actor-Critic with Experience Replay (ACER) \cite{wang2016sample} that make use of experience replay. Although the data from the replay buffer is technically off-policy, the data available to these algorithms is online and comes from previous iterations \cite[Section E, Appendix]{liu2020off}. In contrast, the data available to our algorithm comes from an entirely different (behavior) policy. This is a more challenging setting due to the instability issues that arise in this case. 

\section{Problem Formulation}
We consider an infinite-horizon MDP defined by the tuple \\ $(S,A,P,r,\gamma,d_0)$, where $S$ denotes the state space, $A$ denotes the action space (here we assume that $|A| < \infty$), $P$ is the transition probability rule where $P(s'|s,a)$ denotes the probability of transitioning to state $s'$ when action $a$ is taken in state $s$, $r$ is the single-stage reward function where $r(s,a)$ denotes the reward obtained by the agent from the environment $\mathcal{E}$ when action $a$ is taken in state $s$, $0 < \gamma < 1$ denotes the discount factor and $d_0$ denotes the initial distribution over states. We define a policy $\pi: S \xrightarrow{} \Delta(A)$\footnote{Here $\Delta$ denotes the simplex whose dimension is the cardinality of actions.} as a mapping from state space to distribution over actions. For a given policy $\pi$, we define $J(\pi)$, the discounted long-run reward as follows: 
\begin{align}\label{eq1}
    J(\pi) = (1-\gamma) ~ \E \Bigg[\sum_{t=0}^{\infty} \gamma^t r(s_t,a_t) \Big{|} s_0 \sim d_0 \Bigg],
\end{align}
where the expectation is taken over the states and actions encountered at time $t = 0,1,2,\ldots,\infty$, with $a_t \sim \pi(s_t,.)$. The objective of an RL agent is to compute a policy $\pi^*$ such that,
\begin{align}\label{pidef}
    \pi^* = \arg \max_{\pi \in \Pi} J(\pi),
\end{align}
where $\Pi$ is the set of all possible polices. 

To obtain $\pi^*$, the RL agent typically starts with an initial policy and iteratively finds improved policies. If the agent executes the current learned policy at every stage, we refer to such a setting as ``on-policy control''. As discussed earlier, this may not be a favorable setting in many realistic scenarios. Instead, the agent would have to rely on the data generated by a different policy (known as behavior policy) at every stage of policy improvement. This setting is referred to as ``off-policy control''. Let $\mu$ denote the behavior policy using which the data (trajectories/samples) gets generated. An example of such a policy would be to one that takes all actions with equal probability in all the states. 

Note that the set of all stochastic policies $\Pi$ is uncountably infinite which renders the maximization in \eqref{pidef} intractable. This forces one to resort to function approximation techniques \cite{sutton1998introduction}. One popular way of doing this is by parameterizing the policy $\pi$. Let $\tilde \Pi = \{\pi_{\theta}~|~ \theta \in \R^k, ~ k >0 \}$ be a family of policies parameterized by $\theta$. For example, a soft-max policy parameterization $\forall s,a$ is described as:
\begin{align}
    \pi_{\theta}(s,a) = \frac{\exp{(f(\theta,s,a)})}{ \displaystyle \sum_{b \in A}\exp{(f(\theta,s,b)})},
\end{align}
where $f(\theta,s,a)$ is a linear or non-linear function over parameters $\theta$.

The objective of an agent in such a case is to compute an optimal $\theta^* = \displaystyle \arg \max_{\theta} J(\pi_\theta)$. Note that unlike in \eqref{pidef} where the optimization is over all possible stochastic policies, the optimization here is only over the parameterized policy space $\tilde \Pi$.
This can be done by performing gradient ascent as follows: consider an initial $\theta_0$ and update the parameters as:
\begin{align}\label{eq4}
    \theta_{t+1} = \theta_t + \beta_t \nabla_{\theta}J(\theta_t), ~ \forall t \geq 0,
\end{align}
where $\beta_t, ~ t \geq 0$ is the step-size sequence. Therefore, the problem of computing an optimal $\theta^*$\footnote{Note that the gradient ascent algorithm guarantees convergence only to the local optimal solution.} reduces to efficiently estimating the gradient $\nabla_\theta J(\pi_\theta)$ from the available data. We first discuss how this is done in the on-policy setting, i.e., the setting where the data is generated from the same policy $\pi_\theta$ itself.  
\section{On-Policy Actor-Critic Algorithms:} \label{s3}
A significant result in RL is the policy gradient theorem \cite{sutton1998introduction} that gives the following closed form expression for $\nabla_{\theta} J(\theta)$:
\begin{align}\label{eq2}
    \nabla_{\theta} J(\theta) = \displaystyle \E_{s\sim \Tilde{d}_{\pi_\theta}, a \sim \pi_{\theta}}[\nabla_{\theta} \log \pi_{\theta}(s,a)A^{\pi_{\theta}}(s,a)],
\end{align}
where $\displaystyle \Tilde{d}_{\pi}(s) = (1-\gamma)~ \sum_{t=0}^{\infty} \gamma^t \mathbb{P}\{s_t =s \mid s_0 \sim d_0\}$ is the discounted visitation distribution,   $A^{\pi_{\theta}}(s,a) = Q^{\pi_{\theta}}(s,a) - V^{\pi_{\theta}}(s)$ is the advantage value function, $Q^{\pi_{\theta}}(s,a)$ \cite{sutton1998introduction}  is the Q-value corresponding to state $s \in S$ and action $a \in A$, and $V^{\pi_{\theta}}(s)$ \cite{sutton1998introduction} is the value function of policy $\pi_{\theta}$ corresponding to state $s \in S$. Note that in \eqref{eq2}, we only require the projection of advantage function onto $\nabla_{\theta_i} \log \pi_{\theta}(s,a), \forall i \in \{1,2,\ldots,k\}$ rather than the advantage function itself for computing the gradient. Therefore, in \cite{sutton1999policy,konda2000actor} a linear function approximation for $A^{\pi_{\theta}}(s,a)$ has been considered, i.e.,
\begin{align}\label{eq3}
    A^{\pi_{\theta}}(s,a) \approx x^T \nabla_{\theta} \log \pi_{\theta}(s,a),
\end{align}
where $\nabla_{\theta} \log \pi_{\theta}(s,a)$ is the compatible feature \cite{bhatnagar2009natural} corresponding to state $s$ and action $a$. If $x^*$ satisfies the optimality condition for the projection (see \cite{sutton1999policy,konda2000actor}):
\begin{align}
    \E_{s\sim \Tilde{d}_{\pi_\theta}, a \sim \pi_{\theta}}[(A^{\pi_{\theta}}(s,a) -& {x^*}^T \nabla_{\theta} \nonumber \log \pi_{\theta}(s,a)) \\ &\nabla_{\theta}\log \pi_{\theta}(s,a) ]= 0,
\end{align}
then the policy gradient \eqref{eq2} can be rewritten as:
\begin{align}\label{eq5}
    \nabla_{\theta} J(\theta) = F(\theta) x^*,
\end{align}
where 
\begin{align}
    F(\theta) = \displaystyle \E_{s\sim \Tilde{d}_{\pi_\theta}, a \sim \pi_{\theta}}[\nabla_{\theta} \log \pi_{\theta}(s,a) \nabla_{\theta} \log \pi_{\theta}(s,a)^{T}].
\end{align}
Actor-critic algorithms \cite{sutton1998introduction,prabuchandran2016actor} are a popular class of RL algorithms that utilise the policy gradient theorem to compute the optimal policy $\pi_{\theta^*}$. Here, the critic estimates the advantage value function $(x)$ parameters and the actor uses it to improves the policy parameters ($\theta$).  

The policy gradient update rule in \eqref{eq4} does not take into account the curvature information of the objective function. This makes the policy parameter updates sensitive to the learning rates. To tackle this problem natural policy gradients have been considered in the literature \cite{peters2003reinforcement,peters2008natural,bhatnagar2009natural}. Here, the standard gradient $\nabla_{\theta}J(\theta)$ is replaced by the natural gradient:
\begin{align}\label{eq6}
     G^{-1}(\theta)\nabla_{\theta}J(\theta),
\end{align}
where $G(\theta)$ is the Fisher information matrix that encodes curvature information \cite{amari1998natural}.

An important result pertaining to natural policy gradients comes from \cite[Section 3]{peters2003reinforcement}, where it is shown that $G(\theta) = F(\theta)$ for the infinite horizon discounted problem. Therefore, from \eqref{eq5} and \eqref{eq6}, the natural gradient turns out to be optimal critic parameters $x^*$, i.e., 
\begin{align}\label{eq12}
    G^{-1}({\theta})\nabla_\theta J(\theta) = G^{-1}({\theta})F(\theta)x^* = x^*.
\end{align}

We are now ready to propose our on-policy algorithm ``Deep Natural Actor-Critic'' (cf. Algorithm \ref{alg:deepNAC}), that utilizes natural policy gradients.  
The actor update in the proposed algorithm is performed as: 
\begin{align}\label{eq11}
    \theta_{t+1} &= \theta_t + \beta_t G^{-1}(\theta_t)\nabla_{\theta}J(\theta_t) = \theta_t + \alpha_t x_t, ~ \forall t \geq 0,
\end{align}
where $x_t$ satisfies:
\begin{align}\label{eq10}
    \E_{s\sim \Tilde{d}_{\pi_{\theta_t}}, a \sim \pi_{\theta_t}}[(A^{\pi_{\theta_t}}(s,a) &- x_t^T \nabla_{\theta} \log \pi_{\theta_t}(s,a)) \nonumber \\ &\nabla_{\theta} \log \pi_{\theta_t}(s,a) ]= 0.
\end{align}

In order to estimate $x_t$ satisfying \eqref{eq10}, the Algorithm \ref{alg:deepNAC} employs two critics - an advantage critic and a value function critic. The advantage critic solves \eqref{eq10} iteratively (step 12 of Algorithm \ref{alg:deepNAC}) by approximating \footnote{Note that the exact computation of $A^{\pi_{\theta_t}}$ requires computing expectation over all possible next states. Further, the $V^{\pi_{\theta_t}}(.)$ obtained from the critic is only an approximation to the true value function.} $A^{\pi_{\theta_t}}(s,a)$ as $r(s,a) + \gamma V^{\pi_{\theta_t}}(s') - V^{\pi_{\theta_t}}(s)$. The value function critic then employs deep neural network parameterised by $\psi$ to estimate the value function $V^{\pi_{\theta_t}}$ of the policy $\pi_{\theta_t}$. In this work, we utilize Temporal Difference (TD(0)) prediction \footnote{In general, one could also use TD with eligibility traces, i.e., TD($\lambda$), in prediction for trading bias with variance. For simplicity of exposition, we consider TD(0) prediction here.} \cite{sutton1998introduction}, a simple and effective technique to learn the value function (step 10 of Algorithm \ref{alg:deepNAC}). Finally, using the advantage critic weights $x$, the policy parameters $\theta$ is updated according to \eqref{eq11} (step 14 of Algorithm \ref{alg:deepNAC}).

Note that $\nabla_{\theta}\log \pi_{\theta}(s,a)$ is a matrix of neural network weights (whose size depends on the number of layers and neurons in each layer). Let $k$ be total number of neural network weights in the policy network. In step 11 of Algorithm \ref{alg:deepNAC}, we flatten this matrix into a vector that becomes feature vector for the advantage value function estimation (step 12). Finally, we reshape the advantage value function parameters (step 13) to the matrix form for updating the policy network parameters (step 14).   

\begin{algorithm}[h!]
\caption{On-Policy Deep Natural Actor-Critic (Deep NAC)}\label{alg:deepNAC}
\begin{algorithmic}[1]
\State Initialize the policy network parameter $\theta$. 
\State Initialize the value Function network parameter $\psi$. 
\State Initialize the advantage value function parameter $x \in \R^k$. 
\State Select step-size sequences $\alpha_n, \beta_n, ~ \forall n$.
\For{$n = 0,\ldots,\infty$}
\State Initialise $s \sim d_0(.)$
\While{the trajectory has not terminated}
\State Obtain an action $a \sim \pi_\theta(s,.)$.
\State Obtain next state $s'$ and reward $r$ from the environment.
\State $\psi \xleftarrow{} \psi + \alpha_n \big(r+ \gamma V_{\psi}(s') - V_{\psi}(s)\big)\nabla_{\psi}V_{\psi}(s)$
\State $f \xleftarrow{}$ Flatten matrix $\nabla_{\theta}\log \pi_{\theta}(s,a)$ in to a $k \times 1$ vector. 

\State $x \xleftarrow{} x + \alpha_n \big(r + \gamma V_{\psi}(s') - V_{\psi}(s) - x^Tf\big)f$
\State $M \xleftarrow{}$ Reshape $x$ to the neural network weight matrix. 
\State $\theta \xleftarrow{} \theta + \beta_n M$
\State $s \xleftarrow{} s'$
\EndWhile
\EndFor
\end{algorithmic}
\end{algorithm}

\section{Extension to the Off-Policy Setting:}\label{sec4}
Unlike the on-policy setting, the data in ``off-policy control'' is obtained from a given behavior policy $\mu$. Note that the data in this setting is obtained in the form of samples $(s,a,s')$, where $s \sim \Tilde{d}_\mu(.)$, $a \sim \mu(s,.)$ and $s' \sim p(.|s,a)$. 

Now the policy gradient \eqref{eq2} can be rewritten as:
\begin{align}\label{eq7}
    \nabla_{\theta} J(\theta) &= \displaystyle \E_{s\sim \Tilde{d}_{\pi_\theta}, a \sim \pi_{\theta}}[\nabla_{\theta} \log \pi_{\theta}(s,a)A^{\pi_{\theta}}(s,a)] \\ \nonumber
    &= \displaystyle \sum_{s,a} \Tilde{d}_{\pi_\theta}(s) \pi_{\theta}(s,a) \nabla_{\theta} \log \pi_{\theta}(s,a)A^{\pi_{\theta}}(s,a) \\ \nonumber
    &= \displaystyle \sum_{s,a} \Tilde{d}_{\mu}(s) \frac{\Tilde{d}_{\pi_\theta}(s)}{\Tilde{d}_{\mu}(s)} \mu(s,a) \frac{\pi_\theta(s,a)}{\mu(s,a)}  \nabla_{\theta} \log \pi_{\theta}(s,a)A^{\pi_{\theta}}(s,a) \\
    &= \displaystyle \E_{s\sim \Tilde{d}_{\mu}, a \sim \mu}[w_\theta(s) \rho_\theta(s,a) \nabla_{\theta} \log \pi_{\theta}(s,a)A^{\pi_{\theta}}(s,a)] \label{new-eq7}, 
\end{align}
where $w_\theta(s) = \frac{\Tilde{d}_{\pi_\theta}(s)}{\Tilde{d}_{\mu}(s)}$ and $\rho_\theta(s,a) = \frac{\pi_\theta(s,a)}{\mu(s,a)}$ are the importance sampling ratios. It is important to note the crucial difference between \eqref{eq7} and \eqref{new-eq7}. While the expectation in \eqref{eq7} is over the state, action samples from the policy $\pi_\theta$ itself, the expectation in \eqref{new-eq7} is over the behaviour policy $\mu$. This modification allows one to use the state, action samples from the behavior policy during the training of an off-policy actor-critic algorithm. 

While estimating $\rho_{\theta}(s,a)$ is straightforward, estimating the ratio $ w_{\theta_t}(s_t)$ from data samples is non-trivial. To estimate $w_{\theta_t}(s_t)$ from the samples we employ the iterative algorithm proposed in \cite{liu2018breaking}. 
Please refer to the supplementary material for a discussion on this estimation procedure.

We now briefly discuss the off-policy actor-critic algorithm proposed in \cite{liu2020off}. Utilising \eqref{new-eq7} the policy parameters $\theta_t, ~ t \geq 1$ are updated using batch stochastic gradient ascent as:
\begin{align}
    \theta_{t+1} = \theta_t + \beta_t \displaystyle \sum_{s \in \mathcal{B}} \frac{w_{\theta_t}(s)}{z} \rho_{\theta_t}(s,a) \nabla_{\theta} \log \pi_{\theta_t}(s,a)Q^{\pi_{\theta_t}}(s,a),
\end{align}
where $\mathcal{B}$ is the batch of state, action samples derived from policy $\mu$ at time $t$, and \\$z = \frac{1}{|\mathcal{B}|}\sum_{s \in \mathcal{B}}w_{\theta_t}(s)$. They estimate $Q^{\pi_{\theta_t}}$ using non-linear TD($\lambda$) update. The stationary distribution of the Markov chain $(s_t,a_t), ~ t \geq 0$  under the behavior policy $\mu$ is different from the stationary distribution of the Markov chain under the current policy $\pi_{\theta_t}$. This difference needs to be corrected while estimating the value function of the current policy. 
However, their estimation procedure does not correct for the discrepancy in the state distribution (step 14 of Algorithm 1 of \cite{liu2020off}).

In our work, we make this important correction in the value function estimation procedure and utilize it to estimate the advantage value function that uses compatible features. We finally use the advantage critic weights for computing the natural policy gradient. 




\section{Proposed Algorithm}\label{prop}
\begin{algorithm}[h!]
\caption{Off-Policy Deep Natural Actor-Critic (Deep OffNAC)}\label{alg:NAC}
\begin{algorithmic}[1]
\State $\mu:$ Behavior policy. 
\State Initialize the policy network parameter $\theta$. 
\\State Initialize the value function network parameter $\psi$. 
\State Initialize the advantage value function parameter $x \in \R^k$. 
\For{$n = 0,\ldots,\infty$}
\State Initialize $s \sim d_0(.)$
\While{the trajectory has not terminated}
\State Estimate $\hat{w}_\theta$ and $w_{\theta}$ using Algorithms 1 and 2, \hspace*{0.9 cm} respectively, of \cite{liu2018breaking}.
\State Obtain an action $a \sim \mu(s,.)$.
\State Obtain next state $s'$ and reward $r$ from the environment.
\State Set $\rho(s,a) = \frac{\pi_{\theta}(s,a)}{\mu(s,a)}$.
\State $\psi \xleftarrow{} \psi + \alpha_n \Big( \hat{w}_{\theta}(s)\rho(s,a)\big(r+ \gamma V_{\psi}(s') - V_{\psi}(s)\big)\nabla_{\psi}V_{\psi}(s)\Big)$
\State $f \xleftarrow{}$ Flatten matrix $\nabla_{\theta}\log \pi_{\theta}(s,a)$ in to a $k \times 1$ vector. 
\State $x \xleftarrow{} x + \alpha_n \Big(w_{\theta}(s) \rho(s,a) \big(r + \gamma V_{\psi}(s') - V_{\psi}(s) - x^Tf\big)f\Big) $
\State $M \xleftarrow{}$ Reshape $x$ to the neural network weight matrix.
\State $\theta \xleftarrow{} \theta + \beta_n M$
\State $s \xleftarrow{} s'$
\EndWhile
\EndFor
\end{algorithmic}

\end{algorithm}
In this section, we alleviate the issue discussed above (see section \ref{sec4}) and propose our on-line convergent off-policy natural actor-critic algorithm. In our algorithm, there are two neural networks and one linear network: a policy neural network that updates the policy parameters $\theta$, a value function neural network parameterized by $\psi$ that estimates the value function corresponding to the policy parameters $\theta$ and a linear advantage value critic that uses compatible features $\nabla_\theta \log \pi_\theta(s,a)$  from the actor to compute the natural policy gradient. Recall that $w_\theta(s) = \frac{\Tilde{d}_{\pi_\theta}(s)}{\Tilde{d}_{\mu}(s)}$ and $\rho_\theta(s,a) = \frac{\pi_\theta(s,a)}{\mu(s,a)}$. Let $\hat{w}_\theta(s) = \frac{d_{\pi_\theta}(s)}{d_{\mu}(s)}$ denote the ratio of stationary distributions of the Markov chain $s_t, ~ t \geq 0$ following policies $\pi_\theta$ and $\mu$ respectively. We assume here that the Markov Chains under $\mu$ and $\pi_\theta, ~ \forall \theta,$ have unique stationary distributions $d_\mu$ and $d_{\pi_\theta},$ respectively.
The pseudo-code of our proposed algorithm is given in Algorithm \ref{alg:NAC}. 
The input to our algorithm is the data samples from the behavior policy. Let the current policy parameters be $\theta$. There are three steps involved in the update of policy parameters at each iteration. 
\begin{figure*}
    \centering
    \includegraphics[scale = 0.18]{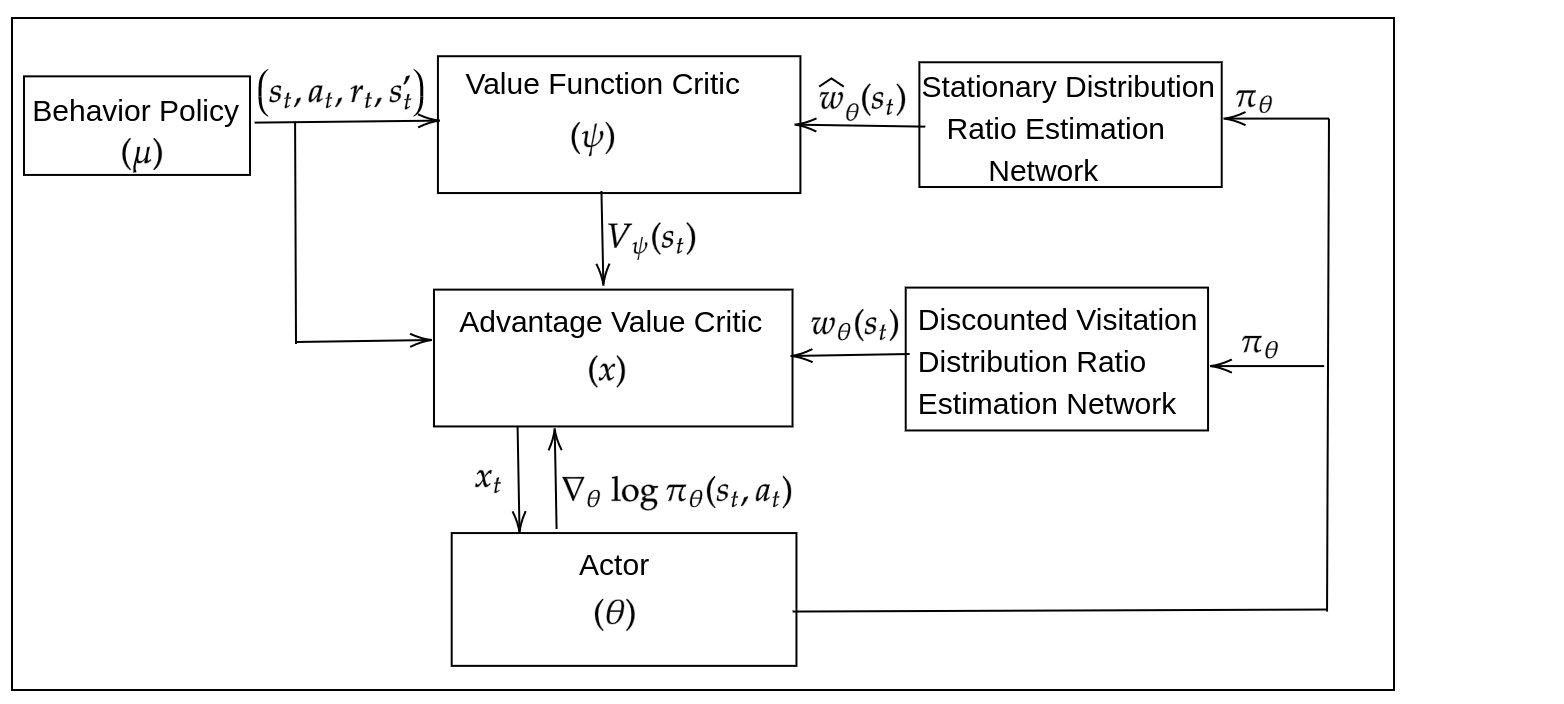}
    \caption{Flowchart of our proposed off-Policy NAC Algorithm}
    \label{fig:PA}
\end{figure*}
These steps are described below:

1. \textit{Estimating the Ratio of Distributions:} The ratio of distributions $\hat{w}_{\theta}$ and $w_{\theta}$ are estimated respectively using Algorithm 1 and Algorithm 2 of \cite{liu2018breaking}. These algorithms take the behavior policy $\mu$ and the current policy $\pi_\theta$ as inputs and estimate the ratios $\hat{w}_{\theta}(s)$ and $w_{\theta}(s)$ for all $s \in S$. The ratio $\hat{w}_{\theta}(s)$ is used by the value function critic and $w_{\theta}(s)$ is used by the advantage value function critic for correcting the state-action distribution mismatch. This will be further explained below.

2. \textit{Estimating the value function:}
Next, the value function corresponding to the current policy is estimated. Under function approximation, this corresponds to computing optimal parameter $\psi^*$.
In on-policy setting, the value function is estimated via the update \cite{sutton1998introduction}:
\begin{align}\label{prop-eq1}
    \psi_{t+1} = \psi_t + \alpha_t (r_t + \gamma V_{\psi_t}(s'_t) - V_{\psi_t}(s_t))&\nabla_{\psi}V_{\psi_t}(s_t), \nonumber \\ & ~ \forall t \geq 0, 
\end{align}
where $\alpha_t$ is the step-size and $(s_t,a_t,r_t,s'_t)$ is the data sample derived from policy $\pi_{\theta_t}$. 

However, as the data sample in the off-policy problem is obtained from the behavior policy, i.e., $s_t$ is sampled from $d_\mu(.)$, $a_t$ is sampled from $\mu(s_t,.)$, the update \eqref{prop-eq1} has to be corrected for state and action distribution mismatch. 
Therefore, we correct the update \eqref{prop-eq1} as follows:
\begin{align}\label{prop-e21}
    \psi_{t+1} = \psi_t + \alpha_t   \hat{w}_{\theta}(s_t)\rho(s_t,a_t) &(r_t + \gamma V_{\psi_t}(s'_t) - V_{\psi_t}(s_t)) \nonumber \\ &\nabla_{\psi}V_{\psi_t}(s_t), ~ \forall t \geq 0.
\end{align}
As $t \xrightarrow{} \infty$, $\psi_t \xrightarrow{} \psi^*$, almost surely \footnote{We assume here that the neural network is rich enough to estimate the actual value function.}. The correction in \eqref{prop-e21} is very important as off-policy value function computation without state and action distribution correction can diverge \cite{sutton1998introduction}.
This crucial value function correction was not done in the previous off-policy control algorithms like \cite{liu2020off}. 
 
3. \textit{Estimating $x$ corresponding to policy $\pi_{\theta}$:} As discussed earlier, the parameter $x$ for the policy $\pi_{\theta}$ should satisfy \eqref{eq10}. This can be computed through an iterative update as follows:
\begin{align}\label{prop-eq3}
    x_{t+1} = x_t &+ \alpha_t (r_t + \gamma V_{\psi^*}(s'_t) - V_{\psi^*}(s_t) \nonumber \\ &- x^T \nabla_{\theta} \log \pi_{\theta}(s_t,a_t)) \nabla_{\theta} \log \pi_{\theta}(s_t,a_t), ~ \forall t \geq 0,
\end{align}
where $r_t + \gamma V_{\psi}(s'_t)- V_{\psi}(s_t)$ is used as a sample estimate for $A^{\pi_{\theta}}(s,a)$. 

However, the data samples above are obtained from $\pi_{\theta}$. In the off-policy problem, we correct the update \eqref{prop-eq3} as follows:
\begin{align}\label{xt-iterates}
    x_{t+1} = x_t &+ \alpha_t w_{\theta}(s_t) \rho(s_t,a_t) (r_t + \gamma V_{\psi^*}(s'_t) - V_{\psi^*}(s_t) \nonumber \\ &- x^T \nabla_{\theta} \log \pi_{\theta}(s,a)) \nabla_{\theta} \log \pi_{\theta}(s,a), ~ \forall t\geq0. 
\end{align}
We now state a result on the convergence of $\{x_t\}, ~ t \geq 0$. 
\begin{theorem}\label{nac-thm1}
For a given policy parameter $\theta$, the advantage critic updates $\{x_t, ~ t \geq 0\}$ given by \eqref{xt-iterates},
where $x_0$ is a randomly chosen initial point, converge to $x^*$, where $x^*$ satisfies:
\begin{align}\label{ppa-eq1}
    \E_{s\sim \Tilde{d}_{\pi_{\theta}}, a \sim \pi_{\theta}}[(A^{\pi_{\theta}}(s,a) - &(x^*)^T \nabla_{\theta} \log \pi_{\theta}(s,a)) \nonumber \\ &\nabla_{\theta} \log \pi_{\theta}(s,a) ]= 0.
\end{align}
\end{theorem}
\begin{proof}
Please refer Section \ref{TR} for the proof of Theorem \ref{nac-thm1}.
\end{proof}

4. \textit{Updating the policy parameters:} Finally, as derived in \eqref{eq11}, the current policy parameter $\theta$ is updated using the natural policy gradient $x^*$ as follows:
\begin{align}\label{prop-policyupdate}
    \theta \xleftarrow{} \theta + \beta x^*,
\end{align}
where $\beta$ is the step-size. As $x^*$ satisfies \eqref{ppa-eq1}, it is the natural gradient direction corresponding to policy $\pi_\theta$ and hence the policy parameters converge to a local optimal solution.

In the above discussion, each step needs to be carried out till convergence before updating the next step. That is, value function parameters $\psi_t$ in \eqref{prop-e21} have to be iteratively run until they converge to obtain $\psi^*$. Subsequently, this $\psi^*$ is used in the update of $x_t$ parameters, which must also be run until convergence to compute $x^*$. Finally, the policy parameter can be updated using \eqref{prop-policyupdate}.  However, this procedure is not feasible for implementation and hence we utilize the concepts of two-timescale stochastic approximation \cite{borkar2009stochastic,kushner2012stochastic} to carry out all the steps sequentially with step-size schedules $\alpha_n, \beta_n, ~ n \geq 0$ in steps 12-16 of the Algorithm \ref{alg:NAC} satisfying
$$\displaystyle \sum_{n= 0}^{\infty} \alpha_n = \sum_{n= 0}^{\infty} \beta_n = \infty, ~ \displaystyle \sum_{n= 0}^{\infty} (\alpha_n + \beta_n)^2 < \infty, ~ \frac{\beta_n}{\alpha_n} \xrightarrow{} 0.$$
The flow chart of an iteration of our algorithm is illustrated in Figure \ref{fig:PA}. Note that the value function $\psi_t$ and natural gradient $x_t$ updates are carried out on a faster time-scale and the policy update $\theta_t$ is carried out on a slower time-scale to ensure convergence.

\section{Theoretical Result}\label{TR}
In this section, we present the proof of Theorem \ref{nac-thm1}. To prove Theorem \ref{nac-thm1}, we cast the $x_t$ updates \eqref{prop-eq3} in the framework of stochastic approximation. In section \ref{ss1}, we briefly describe the standard stochastic approximation scheme \cite[Chapters 2,3]{borkar2009stochastic} and assumptions required for convergence. We then give the convergence proof of iterates $x_t$ followed by the proof of Theorem 1 in section \ref{ss2}.
\subsection{Stochastic Approximation scheme}\label{ss1}
Suppose there is an update of following form:
\begin{align} \label{nac-conv-1}
    &x_{n+1} = x_n + a(n) [h(x_n)+ M_{n+1}],
\end{align}
with a prescribed initial $x_0 \in \R^d$, where 
\begin{fleqn}[0in]
\begin{align*}
    \textbf{(C1)}~ \text{The map $h: \R^d \xrightarrow{} \R^d$ is Lipschitz.}
\end{align*}
\begin{align*}
    \textbf{(C2)}~ \text{The step-size sequence $\{a(n)\}$ satisfy:}\\ \sum_{n= 0}^{\infty} a(n) =\infty,  ~ \sum_{n= 0}^{\infty} a(n)^2 < \infty.
\end{align*}
\begin{align*}
    \textbf{(C3)}~ \text{$\{M_n\}$ is a martingale difference sequence w.r.t increasing}~\\ \text{sequence of sigma-fields:} \mathcal{F}_n = \sigma \{x_m, M_m, ~ m \leq n\}, ~ n \geq 0.
\end{align*}
\begin{align*}
    \textbf{(C4)} ~ \E[\|M_{n+1}\|^2 \mid \F_n] \leq K (1 + \|x_n\|^2), ~ \forall n,\\\text{for some} ~ 0 < K < \infty, ~ \text{almost surely.}
\end{align*}
    \textbf{(C5)} \text{The functions $h_c(x) \doteq \frac{h(cx)}{c}, ~ c \geq 1$} satisfy $h_c(x) \xrightarrow{} h_\infty(x)$ as $c \xrightarrow{} \infty$, uniformly on compacts for some $h_\infty$. Furthermore, the ODE $\dot x(t) = h_\infty(x(t))$ has origin as its unique globally asymptotic stable equilibrium. 
\end{fleqn}

Under \textbf{(C1-C5)}, we have \cite[Theorem 2, Chapter 2]{borkar2009stochastic},
\begin{theorem}\label{nac-sa-thm1}
Almost surely, the sequence $\{x_n\}$ generated by \eqref{nac-conv-1} converges to a compact connected internally chain transitive invariant set of the ODE $\dot x(t) = h(x(t))$.
\end{theorem}
\subsection{Proof of Theorem \ref{nac-thm1}}\label{ss2}
For simplicity, we assume that the state and action spaces are finite. The proof can be extended for continuous state space as well.  
Suppose $s_t \sim \Tilde{d_\mu}(.), a_t \sim \mu(s,.), s'_t \sim P(.|s_t,a_t)$ be the sample obtained at time $t$. Let us re-write $x_{t+1}$ in \eqref{xt-iterates} in the form of stochastic approximation scheme defined in \eqref{nac-conv-1}, i.e.,
\begin{align}
    x_{t+1} = x_{t} + \alpha_t \displaystyle (h(x_t) + M_{t+1}),
\end{align}
where $h(x_t) = -\frac{1}{2}\nabla_x \displaystyle \E_{s\sim \Tilde{d}_{\pi_{\theta}}, a \sim \pi_{\theta}}[(A^{\pi_{\theta}}(s,a) - x_t^T \nabla_{\theta}\log \pi_{\theta}(s,a))^2]$, and $M_{t+1} = w_{\theta}(s_t) \rho_{\theta}(s_t,a_t)(r_t + \gamma V_{\psi^*}(s'_t) - V_{\psi^*}(s_t) - x_t^T\nabla_{\theta}\log \pi_{\theta}(s_t,a_t))\nabla_{\theta}\log \pi_{\theta}(s_t,a_t) - h(x_t)$.

Let $\mathcal{F}_t = \sigma \{\theta_m,x_m,M_m, ~ m \leq t\}, ~ t \geq 0,$ be the increasing family of sigma-fields. We will now verify conditions \textbf{(C1-C5)} and invoke Theorem \ref{nac-sa-thm1} to show the convergence. 

Let us define 
\begin{align}
    F(\theta) = \displaystyle \E_{s\sim \Tilde{d}_{\pi_\theta}, a \sim \pi_{\theta}}[\nabla_{\theta} \log \pi_{\theta}(s,a) \nabla_{\theta} \log \pi_{\theta}(s,a)^{T}].
\end{align}
Then, $h(x_t) = \E[A^{\pi_{\theta}}(s,a)\nabla \log \pi_\theta(s,a)] - F(\theta)x_t.$
For any two $x_1,x_2 \in \R^d,$ we have
\begin{align}
    \|h(x_1) - h(x_2)\| \leq \|F(\theta)\|\|x_1 - x_2\|.
\end{align}
Therefore $h$ is a Lipschitz function with $\|F(\theta)\|$ as the Lipschitz constant and hence \textbf{(C1)} is true.  Moreover, the step-size sequence $\alpha_t = \frac{1}{t+1}, ~ t \geq 0,$ satisfies \textbf{(C2)}.

Let us define $\delta_t = r_t + \gamma V_{{\psi^*}}(s'_t) -V_{{\psi^*}}(s_t)$.
Then, we have 
\begin{align*}
    \E[\delta_t \mid s_t,a_t,\theta] &= r(s_t,a_t) + \sum_{s'}p(s'|s_t,a_t) V_{{\psi^*}}(s') - V_{{\psi^*}}(s_t) \\ & =Q^{\pi_\theta}(s_t,a_t) - V_{{\psi^*}}(s_t) \\ &= A^{\pi_{\theta}}(s_t,a_t).
\end{align*}
Therefore, $\E[M_{t+1}\mid \mathcal{F}_t]$
\begin{align*}
    &= \E[w_{\theta}(s_t) \rho_{\theta}(s_t,a_t)(\delta_t- x_t^T\nabla_{\theta}\log \pi_{\theta}(s_t,a_t))\\ &\hspace{1.5cm} \nabla_{\theta}\log \pi_{\theta}(s_t,a_t) - h(x_t) \Big{|} \F_t] \\
    &= \E[w_{\theta}(s_t) \rho_{\theta}(s_t,a_t)(\E[\delta_t \mid s_t,a_t,\theta,\F_t]- x_t^T\nabla_{\theta}\log \pi_{\theta}(s_t,a_t))\\ &\hspace{1.5cm} \nabla_{\theta}\log \pi_{\theta}(s_t,a_t) - h(x_t) \Big{|} \F_t] \\
    \end{align*}
\begin{align*}
    &= \E[w_{\theta}(s_t) \rho_{\theta}(s_t,a_t)(A^{\pi_{\theta}}(s_t,a_t)- x_t^T\nabla_{\theta}\log \pi_{\theta}(s_t,a_t))\\ &\hspace{1.5cm} \nabla_{\theta}\log \pi_{\theta}(s_t,a_t) - h(x_t) \Big{|} \F_t] \\
    & =\E_{s\sim \Tilde{d}_{\mu}, a \sim \mu}[w_{\theta}(s) \rho_{\theta}(s,a)(A^{\pi_{\theta}}(s,a) - x_t^T \nabla_{\theta}\log \pi_{\theta}(s,a))\\ &\hspace{1.5cm} \nabla_{\theta}\log \pi_{\theta}(s,a)]  - h(x_t) \\
    &= \E_{s\sim \Tilde{d}_{\pi_{\theta}}, a \sim \pi_{\theta}}[(A^{\pi_{\theta}}(s,a) - x_t^T \nabla_{\theta}\log \pi_{\theta}(s,a))\\ &\hspace{1.5cm} \nabla_{\theta}\log \pi_{\theta}(s,a)]  - h(x_t)  \\
    &= 0.
\end{align*}
Therefore, $\{M_{t}, ~ t \geq 0\}$ is a Martingale difference sequence. Hence \textbf{(C3)} is proved.

For $0 < K_2,K_3 < \infty$, let 
\begin{align}
    &\displaystyle \max_{s,a} |r(s,a)| \leq K_2, \label{reward-bound} \\ \label{nabla-bound}
    &\displaystyle \max_{s,a} \| \nabla_\theta \log \pi_\theta(s,a) \| \leq K_3.
\end{align}
Then, we have,
\begin{align}\label{delta-eqn}
   |\delta_t| \leq \frac{2K_2}{1-\gamma} =  K_4 ~ (\text{let}), ~ \forall t \geq 0.
\end{align}
Next, using Cauchy–Schwarz inequality, we have,
\begin{align} \label{cs}
    |x_t^T\nabla_{\theta}\log \pi_{\theta}(s,a)| &\leq \|x_t\|\|\nabla_{\theta}\log \pi_{\theta}(s,a)| \nonumber \\ & \leq K_3\|x_t\|, ~ \forall s,a,t
\end{align}
The behavior policy $\mu$ is selected such that there is positive probability of selecting all actions in each state. For example the behavior policy chosen in the experiments is the one that selects all actions with equal probability in every state. Therefore, for some $0 < K_5,K_6 < \infty$, 
\begin{align}
    &\rho_\theta(s,a) = \frac{\pi_\theta(s,a)}{\mu(s,a)} \leq  \frac{1}{\mu(s,a)} \leq \frac{1}{\min_{s,a} \mu(s,a)} = K_5, ~ \forall s,a,\theta.  \\ 
    &w_\theta(s) = \frac{\tilde{d}_{\pi_\theta}(s)}{\tilde{d}_\mu(s)} \leq \frac{1}{\tilde d_\mu(s)} \leq \frac{1}{\min_{s} \tilde d_\mu(s)} = K_6, ~ \forall s,\theta. \label{nac-w-bound}
\end{align}

From \eqref{nabla-bound}-\eqref{nac-w-bound}, it is easy to see 
\begin{align}
    \E[\|M_{t+1} \|^2 \mid \ \F_t] \leq K(1 + \|x_t\|^2), ~ \text{almost surely,}
\end{align}
where $K = (K_5K_6)^2 \max\{2(K_3K_4)^2,2K_3^4\}$ . Hence \textbf{(C4)} is true.

Now, let us construct the functions $h_c(x) = \frac{h(cx)}{c}, ~ c \geq 1.$ 
\begin{align}
    h_c(x) = \frac{\E[A^{\pi_{\theta}}(s,a)\nabla \log \pi_\theta(s,a)]}{c} - F(\theta)x.
\end{align}
It is easy to see that
\begin{align}
    h_c(x) \xrightarrow{} -F(\theta)x, ~ \text{as} ~ c\xrightarrow{} \infty.
\end{align}
Note that $F(\theta)$ is the Fisher information matrix, which is shown to be positive definite \cite{peters2008natural}. Therefore, $\{0\}$ is the unique globally asymptotic stable equilibrium of the ODE $\dot x(t) = -F(\theta)x(t)$, with $V_1(x) = \frac{1}{2}\|x\|^2$ serving as the corresponding Liapunov function. Hence, \textbf{(C5)} is satisfied.

Finally, let us consider the ODE
\begin{align}\label{nac-last-ode}
    \dot x(t) &= h(x(t))\\ &= -\frac{1}{2}\nabla_x \displaystyle \E_{s\sim \Tilde{d}_{\pi_{\theta}}, a \sim \pi_{\theta}}[(A^{\pi_{\theta}}(s,a) - x(t)^T \nabla_{\theta}\log \pi_{\theta}(s,a))^2].
\end{align}

Then, $\{x^*\}$ is a connected internally chain transitive invariant set of the ODE \eqref{nac-last-ode}, with $V_2(x) = \displaystyle \E_{s\sim \Tilde{d}_{\pi_{\theta}}, a \sim \pi_{\theta}}[(A^{\pi_{\theta}}(s,a) - x^T \nabla_{\theta}\log \pi_{\theta}(s,a))^2]$ serving as the corresponding Liapunov function. 

Therefore, from Theorem \ref{nac-sa-thm1}, we have
\begin{align*}
    x_t \xrightarrow{} x^*, ~ \text{almost surely.}
\end{align*}
\section{Experiments and Results}\label{exp}
In this section, we discuss the performance of the proposed on-policy and off-policy algorithms on four RL tasks available in OpenAI gym \cite{1606.01540}. Please refer to the supplementary material for a detailed description on these four RL tasks. We compare our algorithms ``Deep NAC'' (Algorithm \ref{alg:deepNAC}) and ``Deep OffNAC'' (Algorithm \ref{alg:NAC}) with their standard actor-critic counterparts (which we refer to as) ``Deep AC'' in the on-policy setting and ``Deep OffAC'' in the off-policy setting. 
\subsection{Experimental Setup}
We study the performance of the algorithms in the training and testing phases. The objective in the training phase is to measure how fast an algorithm reaches the maximum total reward. This is depicted through plots where the x-axis denotes the total number of episodes utilised for the training and the y-axis denotes the average total reward achieved by the algorithm. 
The average total reward at an episode $i$ is calculated as: Average total reward $(i)$ = 0.9 $\times$ total reward obtained in episode $i$ + 0.1 $\times$ Average total reward $(i-1)$. This metric gives higher weight to the recent episodes and hence is a better indicator of the performance of the policy. Subsequently in the testing phase, the best performing policy (the policy parameters corresponding to the highest average total reward) during the training is chosen and its performance (mean and standard deviation of the total reward) over $1000$ independent episodes is tabulated.
The detailed description of neural network architectures and hyper-parameters for all the algorithms are included in the supplementary material. 
We first discuss results in the on-policy setting and then results obtained in the off-policy setting.
\subsection{On-Policy Results}
\begin{figure}[ht]
\centering
        \begin{subfigure}[b]{0.25\textwidth}
                \includegraphics[width=\linewidth]{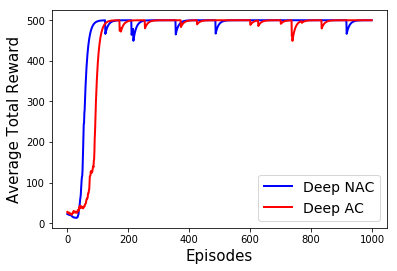}
                \caption{CartPole}
                \label{on1}
        \end{subfigure}%
        \begin{subfigure}[b]{0.25\textwidth}
                \includegraphics[width=\linewidth]{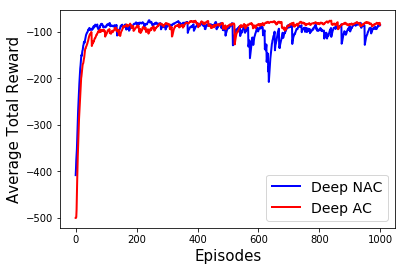}
                \caption{Acrobot}
                \label{on2}
        \end{subfigure}%
        \hfill 
        \begin{subfigure}[b]{0.25\textwidth}
                \includegraphics[width=\linewidth]{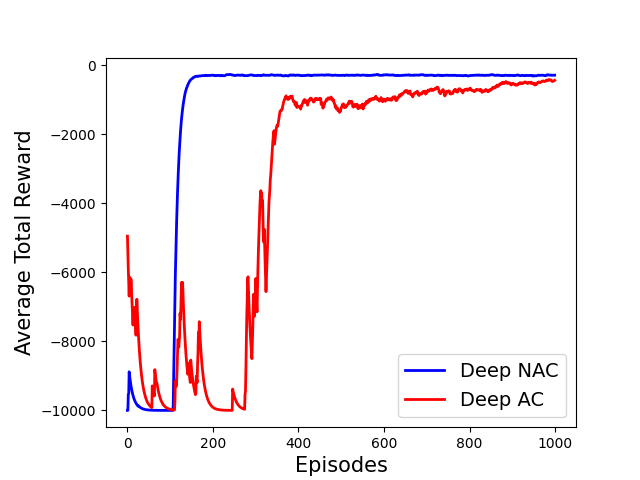}
                \caption{Mountain Car}
                \label{0n3}
        \end{subfigure}%
        \begin{subfigure}[b]{0.25\textwidth}
                \includegraphics[width=\linewidth]{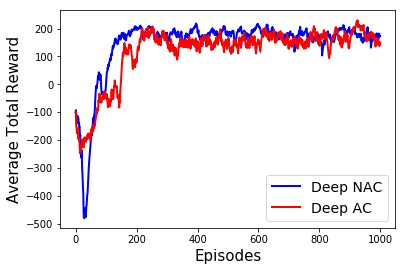}
                \caption{Lunar Lander}
                \label{on4}
        \end{subfigure}%
        \caption{Performance of Deep NAC and Deep AC algorithms during the training phase}\label{on-train}
\end{figure}
\begin{table}[ht]
\centering
\renewcommand{\arraystretch}{1.3}
\begin{tabular}{|c|c|c|}
\hline
\textbf{Environmment} & \textbf{Deep AC} & \textbf{Deep NAC} \\ \hline
\textbf{CartPole}     & 497.20$\pm$6.29   & 498.60$\pm$4.9    \\ \hline
\textbf{Acrobot}      & -83.91$\pm$4.92   & -84.83$\pm$6.31   \\ \hline
\textbf{Mountain Car} & -448.95$\pm$29.10  & -291.13$\pm$8.38  \\ \hline
\textbf{Lunar Lander} & 171.51$\pm$24.14  & 179.93$\pm$15.90  \\ \hline
\end{tabular}
\caption{Comparison between Deep NAC and Deep AC during the testing phase}
\label{on-test}
\end{table}
The performance of algorithms during the training and testing phases is shown in Figure \ref{on-train} and Table \ref{on-test} respectively. From Figure \ref{on-train}, we can observe that our proposed ``Deep NAC'' algorithm reaches the maximum total reward faster compared to the ``Deep AC'' algorithm in all the four tasks. 

The natural gradients employed by the actor in the ``Deep NAC'' algorithm uses the curvature information and hence it learns faster than the ``Deep AC''algorithm. Moreover, as the natural policy gradients are less sensitive to model reparameterization, we observe that hyper-parameter tuning for ``Deep NAC'' is easy compared to ``Deep AC'' algorithm. 
Table \ref{on-test} illustrates the performance comparison in the test phase, i.e., the mean and variance of the best learnt policy on all four tasks.
\subsection{Off-Policy Results}
\begin{figure}[ht]
\centering
        \begin{subfigure}[b]{0.25\textwidth}
                \includegraphics[width=\linewidth]{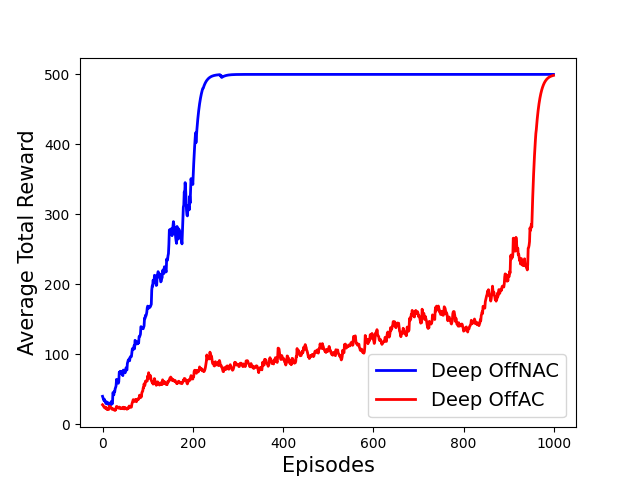}
                \caption{CartPole}
                \label{of1}
        \end{subfigure}%
        \begin{subfigure}[b]{0.25\textwidth}                \includegraphics[width=\linewidth]{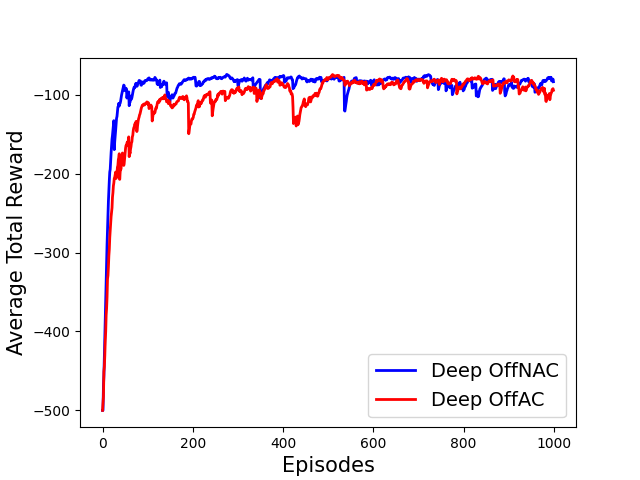}
                \caption{Acrobot}
                \label{of2}
        \end{subfigure}%
        \hfill
        \begin{subfigure}[b]{0.25\textwidth}
                \includegraphics[width=\linewidth]{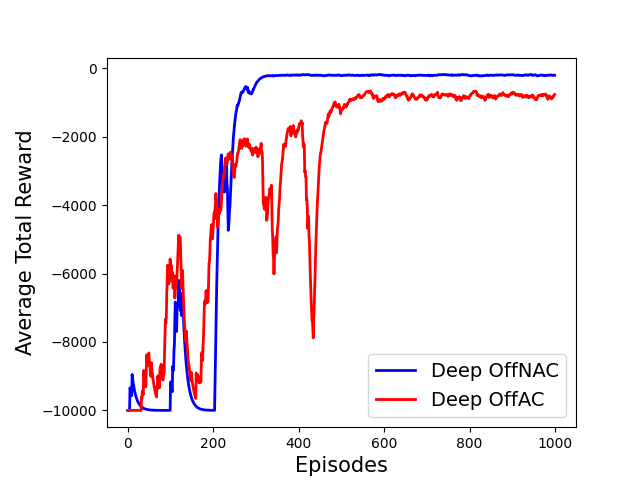}
                \caption{Mountain Car}
                \label{0f3}
        \end{subfigure}%
        \begin{subfigure}[b]{0.25\textwidth}
                \includegraphics[width=\linewidth]{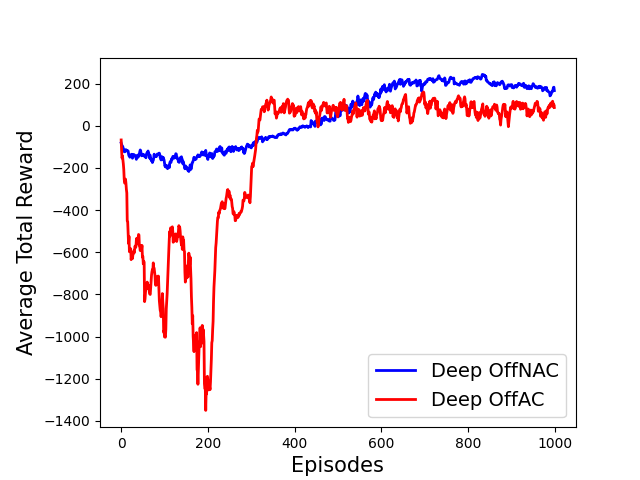}
                \caption{Lunar Lander}
                \label{of4}
        \end{subfigure}%
        \caption{Performance of Deep OffNAC and Deep OffAC algorithms during the training phase}\label{off-train}
\end{figure}
\begin{table}[ht]
\centering
\renewcommand{\arraystretch}{1.3}
\begin{tabular}{|c|c|c|}
\hline
\textbf{Environmment} & \textbf{Deep OffAC} & \textbf{Deep OffNAC} \\ \hline
\textbf{CartPole}     & 499.62$\pm$0.95   & 500$\pm$0    \\ \hline
\textbf{Acrobot}      & -85.33$\pm$5.36   & -84.28$\pm$4.78   \\ \hline
\textbf{Mountain Car} & -874.62$\pm$76.17  & -202.62$\pm$8.85  \\ \hline
\textbf{Lunar Lander} & 87.05$\pm$27.55  & 215.14$\pm$12.97  \\ \hline
\end{tabular}
\caption{Comparison between Deep OffNAC and Deep OffAC during the testing phase}
\label{off-test}
\end{table}
Here, the data during the training is collected from a behavior policy $\mu$ where all actions are uniformly chosen in every state. In order to evaluate the performance during the training, we run an episode with the current learned parameters after every training episode. 
Similar to the on-policy results we see from Figure \ref{off-train} and Table \ref{off-test} that our proposed ``Deep OffNAC'' algorithm reaches the maximum total reward faster in the training phase. Moreover, the best policy obtained from the ``Deep OffNAC'' gives higher reward during the testing compared to the ``Deep OffAC'' algorithm. 
\begin{figure}[!h]
\centering
        \begin{subfigure}[b]{0.25\textwidth}
                \includegraphics[width=\linewidth]{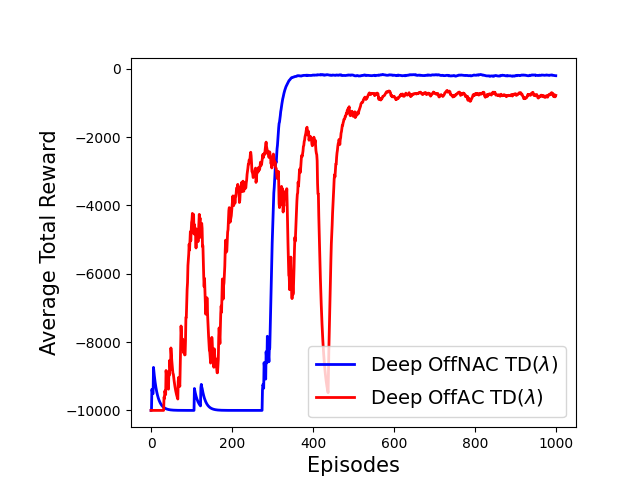}
                \caption{On-Policy }
                \label{ofl1}
        \end{subfigure}%
        \begin{subfigure}[b]{0.25\textwidth}                \includegraphics[width=\linewidth]{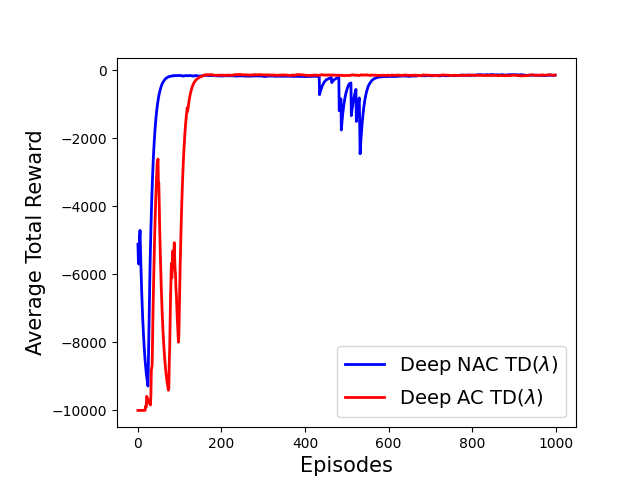}
                \caption{Off-Policy}
                \label{ofl2}
        \end{subfigure}%
\caption{Comparison between Deep AC and Deep NAC with TD($\lambda$) critic during the training phase}
\label{lam-train}
\end{figure}
\begin{table}[ht]
\centering
\renewcommand{\arraystretch}{1.3}
\begin{tabular}{|c|c|c|}
\hline
\textbf{Mountain Car} & \textbf{\begin{tabular}[c]{@{}c@{}}Deep AC with \\ TD($\lambda$) Critic\end{tabular}} & \textbf{\begin{tabular}[c]{@{}c@{}}Deep NAC with \\ TD($\lambda$) Critic\end{tabular}} \\ \hline
\textbf{On-Policy}             & -147.70$\pm$7.39                                                                                     & -143.29$\pm$6.02                                                                                      \\ \hline
\textbf{Off-Policy}            & -771.47$\pm$56.07                                                                                     & -192.66$\pm$9.00                                                                                      \\ \hline
\end{tabular}
\caption{Comparison between Deep NAC and Deep AC algorithms with TD ($\lambda$) critic on Mountain Car}
\label{test-lambda}
\end{table}

From Tables \ref{on-test} and \ref{off-test}, we see that the average rewards achieved by both the algorithms for the Mountain Car task are not satisfactorily high. As this task suffers from delayed reward problem (it receives $0$ only when it reaches the top and $-1$ at intermediate steps), TD(0), due to the high bias nature is unable to estimate the value function optimally. Therefore we consider TD($\lambda)$ critic in place of TD(0) for optimally estimating the value function \cite{sutton1998introduction}. In Figure \ref{lam-train} and Table \ref{test-lambda}, we show the results of Mountain Car trained using TD($\lambda$) critic in both on-policy and off-policy settings during training and testing, respectively. We observe that the average rewards have increased in both on-policy and off-policy settings. Further, the performance of Deep NAC and Deep OffNAC are better than their standard actor-critic counterparts.

\section{Conclusions and Future Work}\label{conc}
In this work, we have proposed the first on-policy deep actor-critic algorithm that utilizes natural gradients in the policy update and compatible features in the advantage critic update for faster convergence. We then extended this algorithm to the off-policy control and proposed a convergent algorithm. 
We illustrated the advantages of the proposed algorithms by comparing it with the standard actor-critic algorithms. 

In the future, we would like to perform ablation studies to better understand the role of the ratio of state-action distributions in the performance of off-policy algorithms. Moreover, it would be interesting to integrate algorithms like Gradient TD \cite{sutton2009fast}, Emphatic TD \cite{sutton2016emphatic} (for value function estimation) with the proposed algorithm. Another important future work is to test the efficacy of our algorithms on Atari games test bed. As these games require deeper neural networks along with convolution neural networks extracting compatible features and implementing natural gradient is a challenging task. 


\section{Acknowledgements}
The authors would like to thank Prof. Koteswararao Kondepu at IIT Dharwad for generously providing GPU support to run our experiments. We also thank High Performance Computing (HPC), IIT Dharwad for  computational resources.


\bibliographystyle{plain} 
\bibliography{references}
\section{Supplementary}\label{suppl}
\subsection{Estimating ratio of State distributions} \label{ratiosection}
In this section, we briefly discuss the procedure for estimating $\hat{w}_\theta$ and $w_\theta$ (proposed in \cite{liu2018breaking}) that we utilize in our proposed ``Deep OffNAC'' Algorithm 2 (Section 5 of the main paper) for correcting state-action distribution .   
Recall that $\hat{w}_\theta(s) = \frac{d_{\pi_\theta}(s)}{d_{\mu}(s)}$ and $w_\theta(s) = \frac{\Tilde{d}_{\pi_\theta}(s)}{\Tilde{d}_{\mu}(s)}$ are the ratio of stationary distributions of the Markov chain $s_t, ~ t \geq 0$ and ratio of state visitation distributions under policies $\pi_\theta$ and $\mu$ respectively. Moreover, $\rho_\theta(s,a) = \frac{\pi_\theta(s,a)}{\mu(s,a)}$ is the importance sampling ratio of the action probabilities.

For any function $f: S \xrightarrow{} \R$ and $w: S \xrightarrow{} \R$, define two functions $L_1$ and $L_2$ as follows:
\begin{align}\label{stat-func}
    L_1(w,f) =\E_{(s,a,s')}[\Delta(w;s,a,s')f(s')],
\end{align}
where the sample $(s,a,s')$ is drawn according to the distribution $d_\mu(s)\times \mu(s,a)\times P(s'|s,a)$ and $\Delta(w;s,a,s') = w(s)\rho_\theta(s,a)- w(s)$. 

\begin{align}\label{l-func}
    L_2(w,f) = \gamma \E_{(s,a,s')}[\Delta(w;s,a,s')f(s')] + (1-\gamma)\E_{s \sim d_0}[(1-w(s))f(s)],
\end{align}
where the sample $(s,a,s')$ is sampled from the distribution $\Tilde{d}_\mu(s)\times \mu(s,a)\times P(s'|s,a)$ and $\Delta(w;s,a,s') = w(s)\rho_\theta(s,a) - w(s)$. 
Note that in \eqref{stat-func}, the state is sampled from stationary distribution $d_\mu$ and in \eqref{l-func}, it is sampled according to state visitation distribution $\title{d}_\mu$. 

It is shown in \cite[Theorem 1]{liu2018breaking} and \cite[Theorem 4]{liu2018breaking} that:
\begin{align}\label{eq9}
    L_1(w,f) = 0, \forall f, ~ \text{iff} ~ w(s) \propto \hat{w}_{\theta}(s), ~ \forall s \in S. \\
    L_2(w,f) = 0, \forall f, ~ \text{iff} ~ w(s) = w_{\theta}(s), ~ \forall s \in S,
\end{align}
respectively. 

Utilising \eqref{eq9}, $\hat{w}_\theta$ could be estimated by solving the following min-max optimization problem:
\begin{align}\label{seq10}
    \displaystyle \min_w \Big{\{} D(w) \coloneqq \max_{f \in \mathcal{F}} L_1(w,f)^2 \Big{\}}, 
\end{align}
where $\mathcal{F}=\{f: S \xrightarrow{} \R$\} is the set of all functions. It is not feasible to directly optimize over all functions $f \in \mathcal{F}$ in \eqref{seq10}. Hence, \cite{liu2018breaking} proposes to optimize over a restricted family of functions like Reproducing Kernel Hilbert space (RKHS) with a suitable kernel. This has an additional advantage of obtaining the closed form expression for $D(w)$ as follows:
\begin{align}\label{seq11}
  D(w) = \E[\Delta(w;s,a,s')\Delta(w;\bar s,\bar a,\bar s')k(s,s')],
\end{align}
where $(s,a,s')$ and $(\bar s,\bar a,\bar s')$ are two independent transition pairs obtained from policy $\mu$ and $k$ is the positive definite kernel. 
If $|S| < \infty$ then using \eqref{seq11}, $\Tilde{w_\theta}$ can be estimated by performing gradient descent over $w$. However, when the number of states are very large or infinite, it is not possible to estimate $\title{w}(s), ~ \forall s$. Hence the function $w: S \xrightarrow{} \R$ can be parameterized (by a neural network) and the gradient descent can be performed over the parameters of the network. By replacing $L_1$ with $L_2$ and obtaining a closed form expression similar to \eqref{seq11} (see \cite[equation 28, Appendix D]{liu2018breaking}), $w_\theta$ could also be estimated in a similar manner.

\subsection{Description of RL tasks and Hyperparameters setting} \label{description}
In this section, we describe the benchmark RL tasks on which we test the algorithms. Further, we also discuss the hyperparameter setting that we set in our experiments.

\subsubsection{CartPole}
It consists of a pole attached by an unactuated joint to a cart. The cart can move on a frictionless track. Episode starts with an upright pole and the task is to balance the pole for maximum number of timesteps by applying either positive or negative force on the cart at each timestep. A \textbf{reward of +1} is received at each timestep until the episode ends. The episode ends if pole is more than $15^{\circ}$ from vertical or cart moves more than 2.4 units from center. The maximum length of episode is \textbf{500 timesteps} and thus the maximum total reward that can be obtained in an episode would be $+500$.
\begin{center}
\includegraphics[width = 8cm]{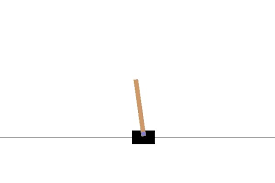}
\end{center}
\textbf{State Space:}
\begin{enumerate}
    \item Cart Position
    \item Cart Velocity
    \item Pole Angle
    \item Pole Angular Velocity
\end{enumerate}
\vspace{1em}
\textbf{Actions:}
\begin{enumerate}
    \item +1 Force on Cart
    \item -1 Force on Cart
\end{enumerate}
\begin{table*}[h]
\centering
\renewcommand{\arraystretch}{1.3}
\hspace*{-1cm}
\begin{tabular}{|c|c|c|c|c|c|c|c|c|c|}
\hline
\multirow{2}{*}{Algorithm} & \multicolumn{4}{c|}{Hidden Layer Neurons(FC Layers)} & \multicolumn{5}{c|}{Learning Rate}\\ \cline{2-10} 
  & Actor & Value Critic & $w$ & $\hat{w}$ & Actor & Advantage Critic & Value Critic & $w$ & $\hat{w}$\\ \hline
\textbf{Deep AC}  & [16] & [64, 64] & - & - & 0.001 & - & 0.005 & - & - \\ \hline
\textbf{Deep NAC} & [16] & [64, 64] & - & - & 0.001 & 0.001 & 0.01 & - & - \\ \hline
\textbf{Deep OffAC} & [16] & [64, 64] & [16] & [16] & 0.0005 & - & 0.01 & 0.001 & 0.001 \\ \hline
\textbf{Deep OffNAC} & [16] & [64, 64] & [16] & [16] & 0.0005 & 0.01 & 0.01 & 0.01 & 0.01 \\ \hline
\end{tabular}
\hspace*{-1cm}
\caption{Hyperparameters for Cart Pole}
\end{table*}
\subsubsection{Acrobot}
In the Acrobot task, we have two links attached to each other. One end of the first link is attached by an unactuated joint to the wall and the other end is attached to the second link by an actuated joint. Episode starts with the links hanging downwards. The task is to swing the lower end of second link at a height atleast the length of one link above the un-actuated joint marked by a black line in minimum number of timesteps by applying torque on the actuated joint. Both the links can swing freely and can pass by each other. A \textbf{reward of -1} is received at each timestep until the episode ends. The maximum length of episode is \textbf{500 timesteps} and the maximum reward that can be obtained in an episode lies in the range between $-90$ to $-80$.
\begin{center}
\includegraphics[width = 8cm]{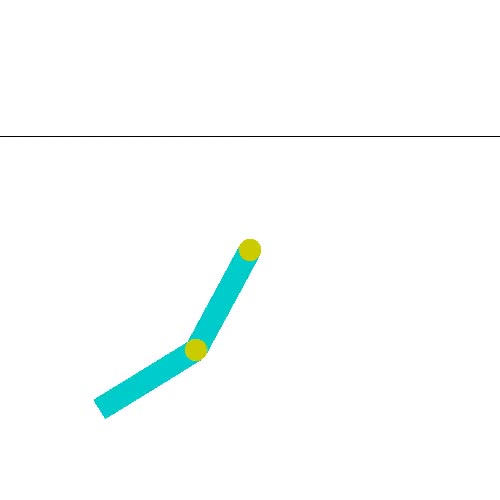}
\end{center}
\textbf{State Space:}
\begin{enumerate}
    \item $cos(\theta_{1})$
    \item $sin(\theta_{1})$
    \item $cos(\theta_{2})$
    \item $sin(\theta_{2})$
    \item Angular velocity of the first link
    \item Angular velocity of second link with respect to first link
\end{enumerate}
$\theta_{1}$ is the angle made by first link with respect to vertically downwards direction. $\theta_{2}$ is the angle made by second link with respect to first link.
\newline
\vspace{1em}
\textbf{Actions:}
\begin{enumerate}
    \item +1 Torque on actuated joint
    \item 0 Torque on actuated joint
    \item -1 Torque on actuated joint
\end{enumerate}
\begin{table*}[h]
\centering
\renewcommand{\arraystretch}{1.3}
\hspace*{-1cm}
\begin{tabular}{|c|c|c|c|c|c|c|c|c|c|}
\hline
\multirow{2}{*}{Algorithm} & \multicolumn{4}{c|}{Hidden Layer Neurons(FC Layers)} & \multicolumn{5}{c|}{Learning Rate}\\ \cline{2-10} 
  & Actor & Value Critic & $w$ & $\hat{w}$ & Actor & Advantage Critic & Value Critic & $w$ & $\hat{w}$\\ \hline
\textbf{Deep AC}  & [32] & [32, 32] & - & - & 0.0005 & - & 0.001 & - & - \\ \hline
\textbf{Deep NAC} & [32] & [32, 32] & - & - & 0.0001 & 0.001 & 0.005 & - & - \\ \hline
\textbf{Deep OffAC} & [32] & [32, 32] & [16] & [16] & 0.0001 & - & 0.005 & 0.0001 & 0.0001 \\ \hline
\textbf{Deep OffNAC} & [32] & [32, 32] & [16] & [16] & 0.00005 & 0.0001 & 0.005 & 0.0001 & 0.0001 \\ \hline
\end{tabular}
\hspace*{-1cm}
\caption{Hyperparameters for Acrobot}
\end{table*}
\subsubsection{Mountain Car}
In this task, a car is positioned on a one-dimensional track positioned between two mountains. Episode starts with car in between the mountains. The task is to drive the car to the top of right mountain which is marked by a flag in minimum number of timesteps using car's engine. The car's engine is not strong enough to scale the mountain in a single pass. A \textbf{reward of -1} is received at each timestep until the episode ends. The maximum length of episode is \textbf{10000 timesteps} and the maximum reward that can be obtained in an episode lies in the range $-110$ to $-120$.
\newline
\vspace{1em}
\begin{center}
\includegraphics[width = 8cm]{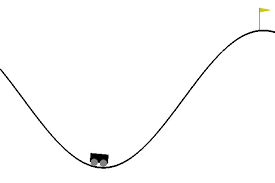}
\end{center}
\textbf{State Space:}
\begin{enumerate}
    \item Car Position
    \item Car Velocity
\end{enumerate}
\vspace{1em}
\textbf{Actions:}
\begin{enumerate}
    \item Accelerate to the left
    \item Don't Accelerate
    \item Accelerate to the right
\end{enumerate}
\begin{table}[h]
\centering
\renewcommand{\arraystretch}{1.3}
\hspace*{-1cm}
\begin{tabular}{|c|c|c|c|c|c|c|c|c|c|}
\hline
\multirow{2}{*}{Algorithm} & \multicolumn{4}{c|}{Hidden Layer Neurons(FC Layers)} & \multicolumn{5}{c|}{Learning Rate}\\ \cline{2-10} 
  & Actor & Value Critic & $w$ & $\hat{w}$ & Actor & Advantage Critic & Value Critic & $w$ & $\hat{w}$\\ \hline
\textbf{Deep AC}  & [32] & [32, 32] & - & - & 0.001 & - & 0.005 & - & - \\ \hline
\textbf{Deep NAC} & [32] & [32, 32] & - & - & 0.00001 & 0.0001 & 0.005 & - & - \\ \hline
\textbf{Deep OffAC} & [32] & [32, 32] & [16] & [16] & 0.0001 & - & 0.005 & 0.01 & 0.01 \\ \hline
\textbf{Deep OffNAC} & [32] & [32, 32] & [16] & [16] & 0.000001 & 0.0001 & 0.005 & 0.01 & 0.01 \\ \hline
\end{tabular}
\hspace*{-1cm}
\\~\\
\hspace*{-2cm}
\begin{tabular}{|c|c|c|c|c|c|c|c|c|c|c|c|}
\hline
\multirow{2}{*}{Algorithm} & \multicolumn{4}{c|}{Hidden Layer Neurons(FC Layers)} & \multicolumn{5}{c|}{Learning Rate} & \multirow{2}{*}{$\lambda$}\\ \cline{2-10} 
  & Actor & Value Critic & $w$ & $\hat{w}$ & Actor & Advantage Critic & Value Critic & $w$ & $\hat{w}$ & \\ \hline
\textbf{Deep AC TD($\lambda$)}  & [32] & [32, 32] & - & - & 0.005 & - & 0.05 & - & - & 0.7\\ \hline
\textbf{Deep NAC  TD($\lambda$)} & [32] & [32, 32] & - & - & 0.0001 & 0.001 & 0.05 & - & - & 1 \\ \hline
\textbf{Deep OffAC TD($\lambda$)} & [32] & [32, 32] & [16] & [16] & 0.0001 & - & 0.005 & 0.01 & 0.01 & 0.7 \\ \hline
\textbf{Deep OffNAC TD($\lambda$)} & [32] & [32, 32] & [16] & [16] & 0.000001 & 0.0001 & 0.005 & 0.01 & 0.01 & 1 \\ \hline
\end{tabular}
\hspace*{-2cm}
\caption{Hyperparameters for Mountain Car}
\end{table}
\subsubsection{LunarLander}
This task consists of lander and a landing pad marked by two flags. Episode starts with the lander moving downwards due to gravity. The task is to land the lander safely using different engines available on lander with zero speed on the landing pad as quickly and fuel efficiently as possible. Fuel is infinite and landing outside the landing pad is also possible. Reward for moving from the top of the screen and landing on landing pad with zero speed is between \textbf{100 to 140 points}. This reward is taken back if lander moves away from landing pad. Each leg ground contact yields a reward of \textbf{10 points}. Firing main engine leads to a reward of \textbf{-0.3 points} in each frame. Firing the side engine leads to a reward of \textbf{-0.03 points} in each frame. An additional reward of \textbf{-100 or +100 points} is received if the lander crashes or comes to rest respectively which also leads to end of the episode. The maximum length of episode is \textbf{1000 timesteps} and the maximum reward that can be obtained in an episode lies in the range $+150$ to $+200$.
\vspace{1em}
\begin{center}
\includegraphics[width = 8cm]{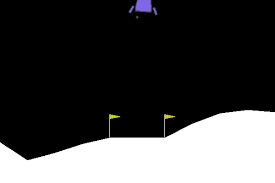}
\end{center}
\textbf{State Space:}
\begin{enumerate}
    \item Horizontal Position
    \item Vertical Position
    \item Horizontal Velocity
    \item Vertical Velocity
    \item Angle
    \item Angular Velocity
    \item Left Leg Contact
    \item Right Leg Contact
\end{enumerate}
\vspace{1em}
\textbf{Actions:}
\begin{enumerate}
    \item Do Nothing
    \item Fire Main Engine
    \item Fire Left Engine
    \item Fire Right Engine
\end{enumerate}
\begin{table}[h]
\centering
\renewcommand{\arraystretch}{1.3}
\hspace*{-1cm}
\begin{tabular}{|c|c|c|c|c|c|c|c|c|c|}
\hline
\multirow{2}{*}{Algorithm} & \multicolumn{4}{c|}{Hidden Layer Neurons(FC Layers)} & \multicolumn{5}{c|}{Learning Rate}\\ \cline{2-10} 
  & Actor & Value Critic & $w$ & $\hat{w}$ & Actor & Advantage Critic & Value Critic & $w$ & $\hat{w}$\\ \hline
\textbf{Deep AC}  & [16] & [64, 64] & - & - & 0.001 & - & 0.005 & - & - \\ \hline
\textbf{Deep NAC} & [16] & [64, 64] & - & - & 0.0001 & 0.001 & 0.005 & - & - \\ \hline
\textbf{Deep OffAC} & [128] & [128, 128] & [16] & [16] & 0.001 & - & 0.01 & 0.0001 & 0.0001 \\ \hline
\textbf{Deep OffNAC} & [128] & [128, 128] & [16] & [16] & 0.00001 & 0.005 & 0.01 & 0.001 & 0.001 \\ \hline
\end{tabular}
\hspace*{-1cm}
\caption{Hyperparameters for Lunar Lander}
\end{table}
\newpage
\subsection{Natural Actor-Critic with TD($\lambda$) Critic}
In this section, we propose a family of algorithms parameterized by $\lambda$, where we replace TD(0) prediction used in the value critic by the TD($\lambda$) prediction \cite{sutton1998introduction}. 
We refer to the new family of algorithms as ``Deep NAC($\lambda$)''. The parameter $\lambda$ helps in trading the bias and variance in the value function estimation of the policy. By choosing a suitable value of $\lambda$ (based on nature of the task), we can improve the prediction of the value function. This in turn can help the actor to compute a better policy. The complete description of the on-policy algorithm is provided in Algorithm \ref{alg:NAC-lam} and the off-policy algorithm is provided in Algorithm \ref{alg:offNAC-lam}.


\begin{algorithm}[h!]
\caption{On-Policy Deep Natural Actor-Critic with TD($\lambda$) prediction (Deep NAC($\lambda$))}\label{alg:NAC-lam}
\textcolor{black}{
\hspace*{\algorithmicindent}\textbf{Input}: $0 \leq \lambda \leq 1$.  \\
\hspace*{\algorithmicindent}Initialize the policy network parameter $\theta$. \\
\hspace*{\algorithmicindent}Initialize the value Function network parameter $\psi$. \\
\hspace*{\algorithmicindent}Initialize the advantage value function parameter $x$. \\
\begin{algorithmic}[1]
\For{$n = 0,\ldots,\infty$}
\State Initialise $s \sim d_0(.)$
\State $z \xleftarrow{} 0$
\While{the trajectory has not terminated}
\State Obtain an action $a \sim \pi_\theta(s,.)$.
\State Obtain next state and reward from the environment $(s',r) \sim \mathcal{E}$.
\State $z \xleftarrow{} \gamma \lambda z + \nabla_{\psi}V_\psi(s)  $ 
\State $\psi \xleftarrow{} \psi + \alpha_n \big(r+ \gamma V_{\psi}(s') - V_{\psi}(s)\big)z$
\State $x \xleftarrow{} x + \alpha_n \big(r + \gamma V_{\psi}(s') - V_{\psi}(s) - x^T\nabla_{\theta}\log \pi_{\theta_t}(s,a)\big)\nabla_{\theta}\log \pi_{\theta}(s,a) $
\State $\theta \xleftarrow{} \theta + \beta_n x$
\State $s \xleftarrow{} s'$
\EndWhile
\EndFor
\end{algorithmic}
}
\end{algorithm}
\begin{algorithm}[h!]
\caption{Off-Policy Deep Natural Actor-Critic with TD($\lambda$) prediction (Deep OffNAC($\lambda$))}\label{alg:offNAC-lam}
\textcolor{black}{
\hspace*{\algorithmicindent}\textbf{Input}: $0 \leq \lambda \leq 1$.  \\
\hspace*{\algorithmicindent}$\mu:$ Behavior policy. \\
\hspace*{\algorithmicindent}Initialize the policy network parameter $\theta$. \\
\hspace*{\algorithmicindent}Initialize the value function network parameter $\psi$. \\
\hspace*{\algorithmicindent}Initialize the advantage value function parameter $x$. 
\begin{algorithmic}[1]
\For{$n = 0,\ldots,\infty$}
\State Initialize $s \sim d_0(.)$
\State $z \xleftarrow{} 0$
\While{the trajectory has not terminated}
\State Estimate $\hat{w}_\theta$ and $w_{\theta}$ using Algorithms 1 and 2, respectively, of \cite{liu2018breaking}.
\State Obtain an action $a \sim \mu(s,.)$.
\State Obtain next state and reward from the environment $(s',r) \sim \mathcal{E}$.
\State Set $\rho(s,a) = \frac{\pi_{\theta}(s,a)}{\mu(s,a)}$.
\State $z \xleftarrow{} \gamma \lambda z + \nabla_{\psi}V_\psi(s)  $ 
\State $\psi \xleftarrow{} \psi + \alpha_n \Big( \hat{w}_{\theta}(s)\rho(s,a)\big(r+ \gamma V_{\psi}(s') - V_{\psi}(s)\big) z\Big)$
\State $x \xleftarrow{} x + \alpha_n \Big(w_{\theta}(s) \rho(s,a) \big(r + \gamma V_{\psi}(s') - V_{\psi}(s) - x^T\nabla_{\theta}\log \pi_{\theta_t}(s,a)\big)\nabla_{ \theta}\log \pi_{\theta}(s,a)\Big) $
\State $\theta \xleftarrow{} \theta + \beta_n x$
\State $s \xleftarrow{} s'$
\EndWhile
\EndFor
\end{algorithmic}
}
\end{algorithm}

\end{document}